\documentclass{article} 
\usepackage{iclr2027_conference,times}

\usepackage{amsmath,amsfonts,bm}

\def\eqref#1{equation~\ref{#1}}

\def\1{\bm{1}}

\DeclareMathAlphabet{\mathsfit}{\encodingdefault}{\sfdefault}{m}{sl}
\SetMathAlphabet{\mathsfit}{bold}{\encodingdefault}{\sfdefault}{bx}{n}

\usepackage{hyperref}
\usepackage{url}

\usepackage{booktabs}
\usepackage{multirow}
\usepackage{makecell}
\usepackage[table]{xcolor}
\usepackage{threeparttable}

\usepackage[nopatch=footnote]{microtype}
\usepackage{hyperref}
\usepackage{url}
\usepackage{booktabs}

\usepackage{lineno}
\usepackage{float}
\usepackage{graphicx}
\usepackage{url}
\usepackage[normalem]{ulem}
\usepackage{enumitem}
\usepackage{multirow}
\usepackage{soul}
\usepackage{xcolor}
\usepackage{subcaption}
\usepackage{amsmath}
\usepackage{amssymb}
\usepackage{mathtools}
\usepackage{amsthm}
\usepackage{caption}
\usepackage{wrapfig}
\usepackage{booktabs}
\usepackage{siunitx}
\usepackage{booktabs}
\usepackage{longtable}
\usepackage{tabularx}
\usepackage{soul}
\usepackage{booktabs}
\usepackage{multirow}
\usepackage{siunitx}
\usepackage{setspace}
\usepackage{amsthm}
\usepackage{diagbox}

\newtheorem{proposition}{Proposition}

\usepackage{wrapfig}
\usepackage{booktabs}

\definecolor{darkgreen}{RGB}{0,145,65}
\definecolor{darkred}{RGB}{190,45,45}
\definecolor{siggreen}{RGB}{0, 100, 0}     
\definecolor{nomgreen}{RGB}{0, 153, 0}  
\definecolor{nomred}{RGB}{230, 0, 0}

\usepackage{algorithm2e, amsmath}
\usepackage{amsthm}

\theoremstyle{definition}
\newtheorem{definition}{Definition}

\usepackage{lineno}

\definecolor{darkblue}{rgb}{0, 0, 0.5}
\hypersetup{colorlinks=true, citecolor=darkblue, linkcolor=darkblue, urlcolor=darkblue}

\title{DeepImagine: Clinical Trial Outcome Prediction via Stepwise Local Counterfactual Imaginations}

\author{%
  Youze Zheng$^{1}$\thanks{Equal contribution.}\quad
  Jianyou Wang$^{1}$\footnotemark[1]\quad
  Yuhan Chen$^{1}$\footnotemark[1]\quad
  Matthew Feng$^{1}$\quad
  Longtian Bao$^{1}$\quad\\
  \textbf{Hanyuan Zhang$^{1}$\quad
  Yang Zhang$^{1}$\quad
  Maxim Khan$^{2}$\quad
  Aditya K. Sehgal$^{2}$}\\
  \textbf{Christopher D. Rosin$^{2}$\quad
  Umber Dube$^{1,3}$\quad
  Ramamohan Paturi$^{1}$}\\[0.4em]
  $^{1}$Laboratory for Emerging Intelligence, UC San Diego\quad  $^{2}$Elsevier\\
  $^{3}$Department of Dermatology, UC San Diego\\
  \texttt{\{yoz018,jiw101,chc248\}@ucsd.edu}
}

\usepackage{placeins}

\usepackage{amsthm}

\theoremstyle{definition}
\newtheorem*{restatedproposition}{Proposition (restated)}
\theoremstyle{plain}

\iclrfinalcopy 
\begin{document}

\maketitle

\begin{abstract}
Predicting the outcomes of prospective clinical trials remains a major challenge. Clinical trial outcomes result from complex interactions among experimental factors such as drug interventions, participant demographics, and protocols. Here, we introduce DeepImagine, a framework that predicts target trial outcomes through stepwise counterfactual imagination anchored on historical trials with observed results. Starting from a relevant historical trial, DeepImagine sequentially modifies one differing experimental factor at a time. With each step a large language model (LLM) is posed a local counterfactual: how would the current imagined outcome change with this single perturbation? The updated result is carried forward as the input to the next step, until the historical configuration matches the target more closely, yielding the final prediction. Empirically, DeepImagine consistently outperforms traditional clinical trial outcome prediction methods and direct one-step prediction across several off-the-shelf LLMs, with further gains when multiple imagination pathways, initiated from different historical anchors, are aggregated. We also construct natural counterfactuals augmented with synthetic reasoning traces and train a family of specialized language models, each dedicated to learning one factor's local counterfactual transition. Integrating these learned local operators into DeepImagine yields substantial improvements over general-purpose LLM baselines. Our findings position stepwise counterfactual imagination, distinct from both correlational prediction and explicit structural causal modeling, as a promising direction for clinical trial outcome prediction. All code, training scripts, and evaluation scripts are available at \href{https://github.com/deepimagine-counterfactual/DeepImagine}{https://github.com/deepimagine-counterfactual/DeepImagine}.
\end{abstract}

\section{Introduction}
Despite rapid advances in artificial intelligence (AI) and biomedical science, predicting the outcome of a prospective clinical trial remains a notoriously difficult problem. Our analysis of 14{,}451 recent Phase II and Phase III drug trials with posted results on \href{https://clinicaltrials.gov/}{clinicaltrials.gov} reveals that 50.45\% failed to achieve their stated primary objectives, directly affecting 3{,}639{,}346 enrolled participants. Recent work \citep{wang2026ctopen} further shows that AI models perform poorly on this task. To fill this research gap, we introduce a novel framework for both inference and training of language models that substantially improves prediction accuracy.

\begin{figure}[ht]
    \centering
    \includegraphics[width=0.92\linewidth]{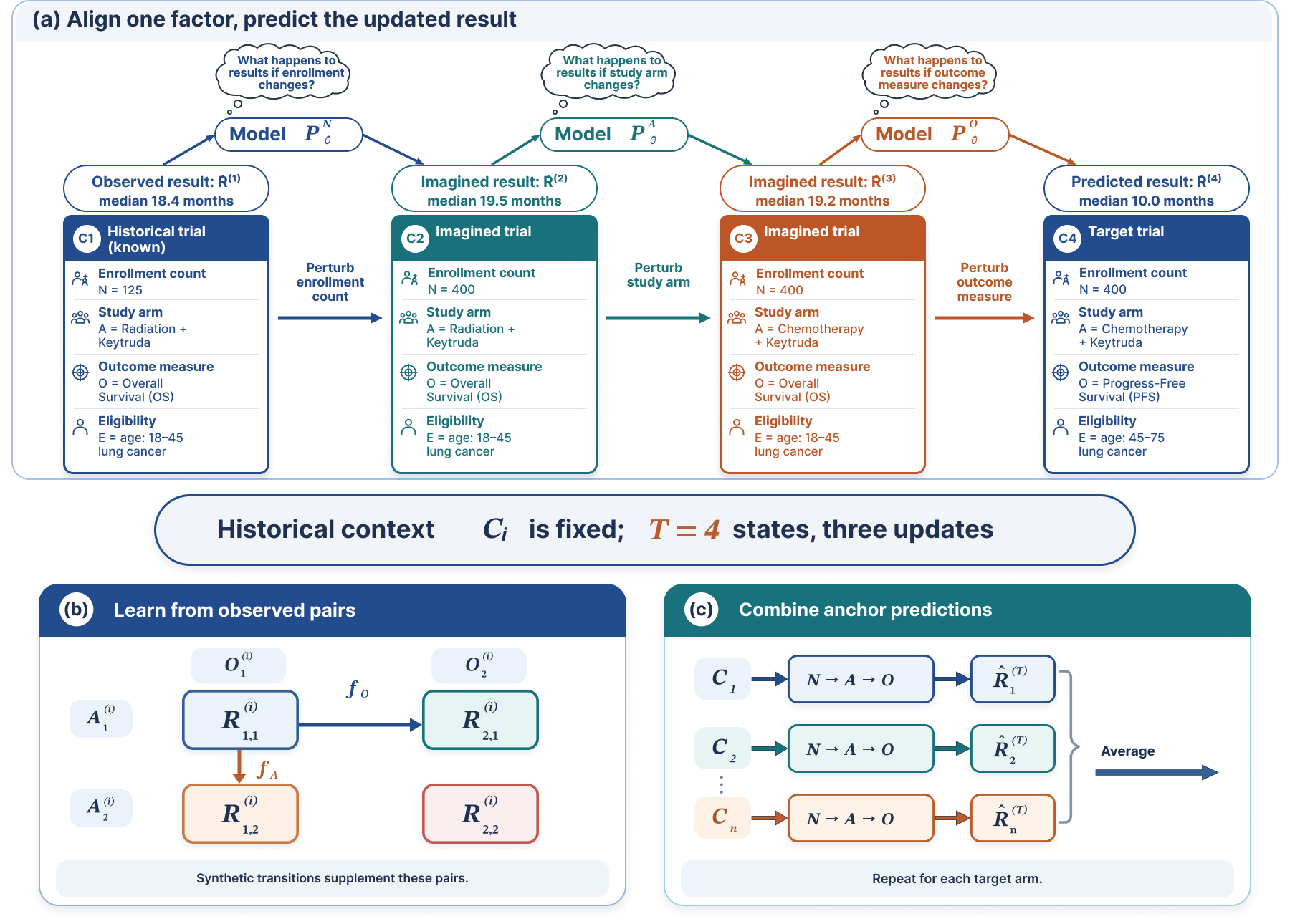}
    \caption{\small DeepImagine uses one local transition for prediction and supervision. \textbf{(a)} A historical result is updated as enrollment, study arm, and endpoint are perturbed towards the target trial. Only the starting result is observed. \textbf{(b)} Within-trial result pairs provide local transitions when other attributes match. \textbf{(c)} Different anchors produce separate predictions for each target arm. Results are averaged.}
    \label{fig:new_new_counterfactual_flow}
\end{figure}

Even a closely related historical trial may differ in its regimen and endpoint. An overall-survival result cannot simply be copied into a target outcome asking about the progression-free survival under another treatment. Retrieval methods provide a relevant historical anchor. The remaining problem is how to adapt that historical observation to the target trial of interest. In practice, trial outcomes depend on intricate interactions among many factors that no closed-form model can fully capture. This likely explains the absence of successful applications of classical causal inference in this setting: the standard pipeline requires specifying a complete structural causal model or learning a causal DAG from data \citep{pearl2009causality, heckerman1999bayesian}, both of which presuppose a tractable causal structure that does not exist here. Direct language-model prediction fares no better, and recent retrieval-augmented (RAG) approaches \citep{wang2026ctopen} that supplement an LLM with similar historical trials yield only minimal and unstable improvements.

Inspired by \cite{wang2026ctopen}, we propose to learn how the observed result of a historical trial would change if its experimental factors were modified to match the target trial's configuration more closely. This bypasses the need for a structural causal model: the model only needs to learn the hidden mechanism governing this counterfactual transition. However, supervised fine-tuning on this task is empirically ineffective because any historical trial  differs from the target along too many factors at once. We therefore propose \textbf{DeepImagine}, which decomposes the prediction into a Markov chain of single-factor perturbations: starting from a historical trial with its observed result, DeepImagine modifies one experimental factor at a time to match the target, and at each step the language model predicts how the current intermediate result would change if only that single factor were perturbed, \textit{ceteris paribus}, i.e., holding all others fixed (Figure~\ref{fig:new_new_counterfactual_flow}). We refer to each step as a {local counterfactual imagination}, using ``counterfactual'' in a procedural sense rather than in the formal sense of \citet{pearl2009causality}. Empirically, we only change three experimental factors: enrollment count $N$, study arm $A$, outcome measure $O$, as we can learn local counterfactual transitions for these factors because each controlled clinical trial already contains naturally occurring counterfactual experiments on them (Section~\ref{sec:training}). We leave learning to change other factors to future work.

Our experiments yield three main findings. First, DeepImagine generalizes effectively to off-the-shelf LLMs such as GPT-5, producing consistent improvements over one-shot RAG-style prediction and traditional clinical trial outcome prediction methods, indicating that even general-purpose LLMs predict more accurately when prompted to reason via local counterfactual imagination. Second, aggregating predictions across multiple pathways anchored on {different} historical trials yields significant additional gains; this is not an artifact of majority voting, as multiple pathways anchored on the {same} trial do not produce the same improvement. Third, language models trained within the DeepImagine framework substantially outperform baselines fine-tuned to predict outcomes directly, suggesting that learning a single-factor transition is far easier than learning the joint transition over all factors at once.

\section{Related Work}

\paragraph{Clinical trial prediction and evaluation.}
Clinical trial prediction spans statistical imputation and structured features \citep{lo2019machine,SIAH2021100312} to models of drug--disease--protocol interactions (HINT; \citealt{fu2022hint}), sequential trial structure (SPOT; \citealt{wang2023spot}), and multimodal representations \citep{zhang2025mexactpmodeexpertscrossattention,zheng2025multimodalclinicaltrialoutcome}. PyTrial provides a broader evaluation toolkit \citep{wang2023pytrial}, and CTP-LLM predicts phase transitions from protocol text \citep{reinisch2024ctpllmclinicaltrialphase}. For data,  we use the Summer 2025 cohort of CT Open \citep{wang2026ctopen}, in which no trial had results released before 2025-08-31; this temporal separation reduces, but does not formally rule out, pretraining contamination.

\paragraph{Grounded, decomposed, and counterfactual reasoning.}
Prior work grounds LLM predictions in retrieved or self-generated examples  \citep{lewis2020rag,wang2022trial2vec, wang2026ctopen, yasunaga2024largelanguagemodelsanalogical}, decomposes problems into simpler or specialized subtasks \citep{zhou2023leasttomostpromptingenablescomplex, khot2023decomposedpromptingmodularapproach, yue2024clinicalagentclinicaltrialmultiagent}, aggregates sampled reasoning paths \citep{wang2023selfconsistencyimproveschainthought}, and supervises reasoning with generated reasoning or step-level correctness labels \citep{zelikman2022starbootstrappingreasoningreasoning, lightman2023letsverify}. DeepImagine is related to these ideas but gives each step concrete state. Starting from an observed historical result, each operator receives the current result and a single configuration change and predicts the next result. We vary the observed starting trial and average arm-level predictions before the final comparison. For supervision, natural counterfactual pairs link two reported trial results through a generated explanation, and synthetic pathways add generated intermediate results without step-level correctness labels. Classical causal inference answers interventional and counterfactual questions under explicit assumptions about a data-generating process \citep{pearl2009causality, heckerman1999bayesian}, and LLM studies examine causal and counterfactual reasoning in various forms \citep{kiciman2023causal, jin2023cladder,miliani2025explica, chen2025counterbench, vashishtha2025improving}. In contrast, DeepImagine is given no causal model, and its single-factor updates are model-based approximations rather than identified causal effects.

\section{Preliminary}
\label{sec:preliminary}

A trial $C_i$ provides study arms, enrollment counts, and outcome measures. We define these quantities and their associated results. Index $i$ identifies a trial, $k$ a study arm, and $j$ an outcome measure.

\begin{definition}[\textbf{Study Arm}]
\label{def:arm}
A study arm is a group of participants assigned to a specific intervention regimen within a clinical trial. Its description can include the interventional drug, dose, and schedule information. A treatment arm receives the target intervention, whereas a comparator arm receives a placebo, standard care, or another intervention. We denote the $k$-th arm of trial $C_i$ by $A_k^{(i)}$.\\
\textit{Example:} A trial may include one arm receiving 200~mg of pembrolizumab every three weeks, and a comparator arm receiving standard chemotherapy. Pembrolizumab's biological mechanism of action is PD-1 blockade which restores T-cell anti-tumor activity via immune cell checkpoint inhibition. 
\end{definition}

\begin{definition}[\textbf{Enrollment Count}]
\label{def:enrollment}
$N_k^{(i)}$ is the number of participants allocated to study arm $A_k^{(i)}$.
\end{definition}

\begin{definition}[\textbf{Outcome Measure}]
\label{def:outcome-measure}
An outcome measure (also called an endpoint) is a pre-specified quantity used to evaluate the effect of a drug intervention. It defines both what will be measured (e.g., 5-year overall survival rate, progression-free survival) and the criterion under which success is declared (e.g., a statistically significant improvement over the comparator at $p<0.05$). We denote the $j$-th outcome measure of trial $C_i$ by $O_j^{(i)}$.\\
\textit{Example:} ``5-year overall survival rate, with success defined as statistically significant improvement over the comparator at $p<0.05$'' is an outcome measure.
\end{definition}

\begin{definition}[\textbf{Trial Result}]
\label{def:result}
A trial result is the observed value of an outcome measure for a particular study arm after the trial has been conducted. For trial $C_i$, outcome $O_j^{(i)}$, and study arm $A_k^{(i)}$ with enrollment $N_k^{(i)}$, we denote the corresponding result by
\[R^{(i)}_{j,k} = \textrm{Result of}\!\left(N_k^{(i)}, A_k^{(i)}, O_j^{(i)}, C_i\right).\]
This representation separates the three explicitly perturbed factors from auxiliary trial context. Trial results are mostly numerical.\\
\textit{Example:} ``85\% of patients in arm $A_k^{(i)}$ survived for at least 5 years'' is a trial result.
\end{definition}

We refer to a set of these factors $X^{(i,j,k)}=\left(N_k^{(i)}, A_k^{(i)}, O_j^{(i)}\right)$ as a \textbf{configuration}. Enrollment is integer-valued whereas arms and outcome measures are textual descriptions. Along a pathway, $(t)$ indexes a state rather than a trial. A historical anchor provides a configuration and its observed result, denoted $(X^{(1)},R^{(1)})$. Its auxiliary trial context $C_i$, including eligibility criteria, is provided separately and held fixed within the pathway. It may differ across anchors.

A CT Open question $(C,O,A_{k_1},A_{k_2})$ asks whether the first study arm will show a statistically significant improvement over its comparator arm at outcome measure $O$ \citep{wang2026ctopen}. The target description and historical observations are available at inference time and the target results are withheld. DeepImagine predicts arm-level results before forecasting this binary comparison. Macro-F1 evaluates the final answers, not intermediate results.

\section{The DeepImagine Framework}
\label{sec:deepimagine}

DeepImagine asks a local question before making a global prediction: given a historical result, what would change if only one experimental factor were different? Figure~\ref{fig:new_new_counterfactual_flow} illustrates why the intermediate result must be carried forward. After enrollment count is updated, the study arm operator reasons about the new enrollment rather than starting from the original trial. Then, the outcome measure operator reasons about the updated regimen. We first formalize these updates, then explain the single-path approximation and multi-anchor aggregation, and finally show how the same updates become supervised training tasks.

\begin{definition}[\textbf{Local Counterfactual Imagination}]
\label{def:local-cf}

Let $\Phi_{t-1}\in \{N,A,O\}$ be the only factor that differs between two clinical trial configurations, $X^{(t-1)}$ and $X^{(t)}$, taking values $\phi_{t-1}$ and $\phi_t$, respectively, within this transition. Given the observed or imagined result $R^{(t-1)}$, a local counterfactual imagination is the mapping
\[
f_{\Phi_{t-1}}:\ \bigl(R^{(t-1)},\, X^{(t-1)},\, X^{(t)}\bigr)\ \mapsto\ {R}^{(t)},
\]
which produces an imagined result $R^{(t)}$ by setting the factor $\Phi_{t-1}$ to be $\phi_t$ instead of $\phi_{t-1}$.
\end{definition}

\begin{definition}[\textbf{Probabilistic Parametrization}]
\label{def:prob-parametrization}
For factor $\Phi_{t-1}$, let $P_\theta^{\Phi_{t-1}}$ be a conditional language model with parameters $\theta$. We use the argmax of a probabilistic language model $P^{\Phi_{t-1}}_{\theta}$ to calculate the local counterfactual imagination operator $f_{\Phi_{t-1}}$:
\[
f_{\Phi_{t-1}}\bigl(R^{(t-1)},\, X^{(t-1)},\, X^{(t)}\bigr) \;=\; \operatorname*{argmax}_{R^{(t)}}\; P_\theta^{\Phi_{t-1}}\!\left(R^{(t)} \,\middle|\, R^{(t-1)},\, X^{(t-1)},\, X^{(t)}\right),
\]
where $\theta$ denotes learnable parameters. The three-factor framework uses one specialized operator for each factor in $\{N, A, O\}$.
\end{definition}

This result-level mode is an idealization: token-wise decoding and result extraction need not recover it. The analysis below requires its stated concentration condition at the retained result Appendix~\ref{app:notation}.

\begin{definition}[\textbf{Stepwise Local Counterfactual Imaginations}]
\label{def:stepwise-cf}
A pathway of stepwise local counterfactual imaginations from source configuration $X^{(1)}$ with observed result $R^{(1)}$ to a target configuration $X^{(T)}$ is a chain
\[
\bigl(X^{(1)},\, R^{(1)}\bigr) \xrightarrow{\;\; f_{\Phi_1} \;\;} \bigl(X^{(2)},\, R^{(2)}\bigr) \xrightarrow{\;\; f_{\Phi_2} \;\;} \cdots \xrightarrow{\;\; f_{\Phi_{T-1}} \;\;} \bigl(X^{(T)},\, R^{(T)}\bigr).
\]
where, for each $t=2,\ldots,T$, the configurations $X^{(t-1)}$ and $X^{(t)}$ differ only in factor $\Phi_{t-1}$, and $R^{(t)}=f_{\Phi_{t-1}}\bigl(R^{(t-1)}, X^{(t-1)}, X^{(t)}\bigr)$ is the result produced by the corresponding local counterfactual imagination. 
\end{definition}

We use $T=4$ and $(\Phi_1, \Phi_2, \Phi_3)=(N, A, O)$, so three updates connect four states. The enrollment update changes only enrollment, keeping the historical study arm and outcome. The arm update then takes the enrollment-updated result as input, and the outcome update takes the arm-updated result. Because each update conditions on the preceding predicted result, later predictions can depend on earlier changes. The intermediate configurations are imagined experiments, rather than additional observed trials. We denote predicted results by $\hat{R}^{(t)}$, with $\hat{R}^{(1)}=R^{(1)}$. The final prediction is $\hat{R}^{(4)}$.

Note that only one factor changes at each step, but its predicted effect depends on the current configuration and result. The composition therefore does not assume additive factor effects nor invariance to the order in which factors are changed. For example, the effect predicted for an arm change can depend on the enrollment already updated and on the outcome specified at that step. Section~\ref{sec:perturbation_order} examines order dependence through an ablation study.

\subsection{Approximating the Marginal with a Single Greedy Pathway}
\label{ssec:single-pathway-approx}

Our key empirical observation underlying DeepImagine is that directly predicting the target result $R^{(T)}$ from $(R^{(1)}, X^{(1)}, X^{(T)})$ in one step is hard, whereas single-factor counterfactual transitions are substantially easier to learn. However, by law of total probability, the marginal for $R^{(T)}$ via stepwise local counterfactual imaginations is intractable as it requires integrating over all possible values for intermediate results $R^{(2)}$ to $R^{(T-1)}$ Equation~\ref{eq:marginal_equation} in Appendix~\ref{app:pathway-math} DeepImagine avoids this integration by single-factor perturbation at each step.

Suppose that, along this single-factor perturbation pathway, each local conditional places mass at least $1-\delta_t$ on $\hat{R}^{(t)}$ Definition~\ref{def:dominant}. The following approximation result characterizes when this single trajectory approximates the composed model.

\begin{proposition}[\textbf{Greedy Approximation of the Marginal}]
\label{thm:single-pathway-approx}
Suppose each $P_\theta^{\Phi_{t-1}}$ is $\delta_t$-dominant for $t = 2, \ldots, T$. The absolute difference between the greedy pathway probability (see Equation~\ref{eq:greedy}) and the model marginal at $\hat{R}^{(T)}$ is at most $\sum_{t=2}^T \delta_t$. For $T=4$, the bound is $\delta_2+\delta_3+\delta_4$.

\end{proposition}

The full inequality and proof of Proposition~\ref{thm:single-pathway-approx} are in Appendix~\ref{app:single-pathway-approx}. The argument uses the chain rule and the elementary inequality $\prod_t(1-\delta_t) \geq 1 - \sum_t \delta_t$. Empirically, this approximation works well and since we keep $T$ relatively small, the chance of error propagation is reduced and the error bound is also tightened.

\subsection{Convergence of Aggregating Multiple Counterfactual Pathways}
\label{ssec:multi-anchor}

Another key empirical finding is that aggregating pathways from different historical anchors improves on single-anchor prediction. We formalize the role of averaging under independent sampling and calibration assumptions.

Different anchors provide different starting observations for the same target $X^{(T)}$. Write $\bar{P}_n$ for the average of $n$ pathway marginals and $M_\theta$ for its expectation under the data-generating law $p_*$. The analysis assumes i.i.d. complete historical records, including their fixed contexts. Calibration requires $M_\theta(\cdot |X^{(T)})=P_{p_*}(\cdot |X^{(T)})$, the true target conditional (Definition~\ref{def:calibration}). Appendix~\ref{app:multi-anchor} gives the complete setup.

\begin{proposition}[\textbf{Convergence of Aggregating Counterfactual Pathways}]
\label{thm:multi-anchor}
Under i.i.d. anchor sampling and calibration, for every fixed target configuration $X^{(T)}$ and fixed measurable event $B\in \mathcal{B}$,
\[
\frac{1}{n} \sum_{i=1}^{n} P\!\left(B \,\middle|\, R_i^{(1)}, X_i^{(1)}, X^{(T)}\right) \;\longrightarrow\; P_{p_\star}\!\left(B \,\middle|\, X^{(T)}\right) \qquad \text{almost surely as } n \to \infty.
\]
\end{proposition}

The strong law makes the average approach $M_\theta$. Calibration identifies that limit with the target law. With $\epsilon$-approximate calibration and the common-density conditions in Appendix~\ref{app:multi-anchor-relaxed}, the aggregate's asymptotic total-variation distance from the target law is at most $\epsilon$. Averaging reduces sampling variance under these assumptions, but does not remove calibration error.

The calibration assumption is also plausible by construction. Our training procedure for the local operators $P_\theta^\Phi$ (Section~\ref{sec:training}) explicitly targets it: the natural counterfactuals used during training are observed conditional realizations of the true law $p_\star$, and the training objective minimizes a divergence between $P_\theta^\Phi$ and the empirical conditional. Calibration is thus not an external condition imposed on the analysis but the operational target of training itself.

\subsection{Training with Natural and Synthetic Counterfactuals}
\label{sec:training}

The local conditional is both an inference operation and a supervised prediction task. We train one operator $P_\theta^\Phi$ for each factor $\Phi\in \{N,A,O\}$ on tuples $(R^{(t-1)},X^{(t-1)},R^{(t)},X^{(t)})$. The preceding result and both configurations are inputs; the updated result is the target. We collect such tuples from two complementary sources: \textbf{natural counterfactuals}, whose paired results are both observed, and \textbf{synthetic counterfactuals}, which include synthetically generated intermediate results.

\textbf{Natural Counterfactuals from a Controlled Clinical Trial:} A controlled clinical trial typically reports several outcome measures across multiple study arms. For one arm $k$ in trial $i$, $R_{j_1,k}^{(i)}$ and $R_{j_2,k}^{(i)}$ are observed under two outcome measures $O_{j_1}^{(i)}$ and $O_{j_2}^{(i)}$. The corresponding configurations differ in the outcome field. This gives a natural pair for learning how the result changes when the measured quantity changes, while keeping the intervention unchanged. We prompt GPT-5 to explain the transition between the two observed results and train $P_\theta^O$ to generate this explanation followed by $R_{j_2,k}^{(i)}$. GPT-5 supplies the explanation, not the result label. Appendix~\ref{app:training-tuples} retains the full conditional training expression.

Analogously, for outcome measure $O_j$ and study arms $k_1,k_2$, the observed results $R_{j,k_1}^{(i)},R_{j,k_2}^{(i)}$ provide input-target pairs for $P_\theta^A$, where the same trial can provide several local tasks rather than one trial-success label. These are aggregate comparisons, not both  potential outcomes for individual patients. Exact single-factor matching requires the remaining attributes, including enrollment and analysis population, to agree, and randomization alone does not ensure this ({Appendix~\ref{app:clinical-interpretation}}).

\textbf{Synthetic Counterfactuals from Similar Trial Pairs.} Natural local counterfactuals are unavailable for the enrollment factor $N$, because constructing one would require running an alternate version of the same trial \textit{ceteris paribus}, varying only the factor of interest. To address this gap, we develop a synthetic counterfactual pipeline. We first embed each historical trial with NV-Embed-v2 \citep{lee2025nvembed}, retrieve candidate pairs based on cosine similarity, and use GPT-5 to verify the semantic similarity of each retrieved pair. In practice, however, even the most similar trials still differ along all three endogenous factors $\{N, A, O\}$, so they cannot be used directly as single-factor local counterfactuals.

Concretely, for historical trials $C_1,C_2$ with configurations $(N^{(1)}, A^{(1)}, O^{(1)})$ and $(N^{(2)}, A^{(2)}, O^{(2)})$ with observed results $R^{(1)}$ and $R^{(2)}$, we prompt GPT-5 to imagine a pathway that changes enrollment, study arm, and outcome measure in that order. Both observed results are provided to guide the generated transitions. The source context is held fixed; the second trial's result is reference evidence, not an observed label for an imagined configuration with potentially different context. Each step and its accompanying explanation yield a training tuple. These tuples provide the enrollment operator's supervision and augment the natural sets for $P_\theta^A$ and $P_\theta^O$. Appendix~\ref{app:training-tuples} gives the complete three-step bridge and distinguishes observed from generated results.

\section{Results and Analyses}
\label{sec:experiments}
\subsection{Experimental Setup}
\label{ssec:exp-setup}

\paragraph{Training data.}
Our data source is $14{,}000$ randomized controlled drug trials with posted results on \href{https://clinicaltrials.gov/}{clinicaltrials.gov}, restricted to trials with enrollment count $N \geq 50$ to filter out underpowered studies. Following Section~\ref{sec:training}, we extracted natural counterfactuals: $57$k for training $P_\theta^A$ and $82$k for training $P_\theta^O$. We also extracted synthetic counterfactuals: $45$k per factor $\Phi \in \{N, A, O\}$. For $P_\theta^A$ and $P_\theta^O$, these augment the natural sets; for $P_\theta^N$, no natural counterfactual is available, so only synthetic data is used. GPT-5 generates the accompanying explanations. We hold out $10\%$ of each operator's training set for hyperparameter tuning. Appendix~\ref{app:training-details} reports the full settings.

To assess the quality of our training data, a biomedical expert reviewed 90 generated explanations: 87 were clinically sound, two were potentially problematic, and one was uncertain. Most counterfactual results are therefore derived through sound biomedical reasoning, although a lower percentage of flawed explanations add some noise to training.

\paragraph{Evaluation benchmarks.}
We evaluate the time-stamped CT Open benchmark~\citep{wang2026ctopen}: \textit{Summer 2025} (cutoff 2025-08-31) with 701 questions including 442 positives and 259 negatives. Each question is a binary decision over a tuple $(C, O, A_{k_1}, A_{k_2})$ asking whether $R_{k_1}$ shows a statistically significant improvement over $R_{k_2}$. DeepImagine produces separate estimates $\hat{R}_{k_1}$ and $\hat{R}_{k_2}$ via stepwise local counterfactual imagination, and for the off-the-shelf LLM backbones, the same LLM converts these predictions into a binary answer. We report macro-F1, the unweighted mean of the two class-specific F1 scores, on a 0--100 scale.

\paragraph{Models.} We evaluate three off-the-shelf LLMs as DeepImagine backbones: GPT-5-2025-08-07, GPT-5-mini-2025-08-07, o3-mini-2025-01-31. Their stated knowledge cutoffs precede the benchmark cutoff. We instantiate each $P_\theta^\Phi$, for $\Phi \in \{N, A, O\}$, as a separate  Gemma-3-12B-it model \citep{gemma3report}, fine-tuned on its corresponding training set. We also evaluate six conventional classifiers \citep{lo2019machine,tranchevent2019deep,Fan2020KPCA,SIAH2021100312} and trial prediction methods including HINT \citep{fu2022hint}, SPOT \citep{wang2023spot}, LIFTED \citep{zheng2025multimodalclinicaltrialoutcome}, and MEXA-CTP \citep{zhang2025mexactpmodeexpertscrossattention}. The conventional models combine numerical, categorical, and NV-Embed-v2 text features. {Appendix~\ref{app:baselines}} gives full details about data processing, method adaptation, and hyperparameters.

\begin{table}[t]
  \centering
  \caption{\small Prediction strategies and trained predictors on the 701-question CT Open cohort (macro-F1, 0--100). Values after $\pm$ are reported 95\% bootstrap CI half-widths. Bold marks the largest GPT-row or baseline-family estimate (Appendix~\ref{app:baselines}). Relative to RAG within each backbone, \textcolor{darkgreen}{$\uparrow$} and \textcolor{darkred}{$\downarrow$} indicate nominal increases and decreases, respectively, while \textcolor{darkgreen}{$\uparrow\uparrow$} denotes statistically significant improvements.}
  \label{tab:main}
 
  \setlength{\tabcolsep}{4pt}
  \begin{tabular*}{\linewidth}{@{\extracolsep{\fill}}lccccc@{}}
    \toprule
    \textbf{A. LLMs} & \multicolumn{2}{c}{One-step prediction} & \multicolumn{3}{c}{DeepImagine} \\
    \cmidrule(lr){2-3} \cmidrule(l){4-6}
    Backbone & RAG & Direct
      & \begin{tabular}[c]{@{}c@{}}1 anchor\\1 pathway\end{tabular}
      & \begin{tabular}[c]{@{}c@{}}1 anchor\\7 pathways\end{tabular}
      & \begin{tabular}[c]{@{}c@{}}7 anchors\\1 pathway each\end{tabular} \\
    \midrule
    GPT-5      & 72.4$\pm$3.32 & 69.4$\pm$3.41\hspace{0.25em}\textcolor{darkred}{$\downarrow$} & 77.1$\pm$3.14\hspace{0.25em}\textcolor{darkgreen}{$\uparrow$} & 75.8$\pm$2.99\hspace{0.25em}\textcolor{darkgreen}{$\uparrow$} & 
    \makebox[5.3em][r]{\textbf{79.0}$\pm$3.09}\makebox[1.5em][c]{\textcolor{darkgreen}{$\uparrow\uparrow$}}\\
    GPT-5-mini & 71.7$\pm$3.38 & 72.8$\pm$3.34\hspace{0.25em}\textcolor{darkgreen}{$\uparrow$} & 73.3$\pm$3.31\hspace{0.25em}\textcolor{darkgreen}{$\uparrow$} & 76.0$\pm$3.16\hspace{0.25em}\textcolor{darkgreen}{$\uparrow$} & 
    \makebox[5.3em][r]{\textbf{78.5}$\pm$3.06}\makebox[1.5em][c]{\textcolor{darkgreen}{$\uparrow\uparrow$}}\\
    o3-mini & 71.0$\pm$3.31 & 73.1$\pm$3.38\hspace{0.25em}\textcolor{darkgreen}{$\uparrow$} & 72.9$\pm$3.27\hspace{0.25em}\textcolor{darkgreen}{$\uparrow$} & 72.6$\pm$3.43\hspace{0.25em}\textcolor{darkgreen}{$\uparrow$} & 
    \makebox[5.3em][r]{75.6$\pm$3.34}\makebox[1.5em][c]{\textcolor{darkgreen}{$\uparrow$}}\\
    \bottomrule
  \end{tabular*}
 
  \vspace{0.6em}

  \begin{minipage}[t]{0.47\linewidth}
    \begin{tabular*}{\linewidth}[t]{@{\extracolsep{\fill}}lr@{}}
      \toprule
      \textbf{B. Conventional classifiers} & macro-F1 \\
      \midrule
      PLR ($\ell_2$)        & \textbf{70.8} \\
      Feed-forward NN       & 66.4 \\
      XGBoost               & 66.2 \\
      Random forest         & 65.9 \\
      $k$NN imputation + RF & 60.3 \\
      AdaBoost              & 59.8 \\

      \addlinespace[\dimexpr\aboverulesep+\lightrulewidth+\belowrulesep\relax]
      \bottomrule
    \end{tabular*}
  \end{minipage}\hfill
  \begin{minipage}[t]{0.47\linewidth}
    \begin{tabular*}{\linewidth}[t]{@{\extracolsep{\fill}}lr@{}}
      \toprule
      \textbf{C. Clinical predictors} & macro-F1 \\
      \midrule
      MEXA-CTP & \textbf{69.8} \\
      HINT     & 59.5 \\
      SPOT     & 59.3 \\
      LIFTED   & 58.6 \\
      \midrule
      \textbf{Trained local operators} & \\
      DeepImagine (Gemma) & \textbf{78.7} \\
      \bottomrule
    \end{tabular*}
  \end{minipage}
\end{table}

For off-the-shelf LLMs, we compare four strategies:

\textbf{Direct:} The LLM sees only the target tuple $(C, O, A_{k_1}, A_{k_2})$ before giving a binary answer.\\\
\textbf{RAG.} For each arm $A_{k}$, we retrieve up to $7$ historical trials (with observed results) that are most relevant to that arm. The LLM sees these retrieved results and the target configuration before giving an answer. This is the standard RAG baseline of \citet{wang2026ctopen}.\\
\textbf{DeepImagine ($n$-anchor).} For each arm, $1\leq n \leq 7$ distinct retrieved trials initialize separate three-update pathways. RAG and DeepImagine use the same retrieved evidence. The $n$ predictions are averaged, following Proposition~\ref{thm:multi-anchor}.\\
\textbf{DeepImagine (same-anchor, $n$ pathways).} An ablation in which all $n$ pathways start from the same historical trial. This violates the i.i.d.\ anchor assumption of Proposition~\ref{thm:multi-anchor} and controls for whether multi-anchor improvement can be attributed to majority voting alone.

\subsection{Ablation Study Setup}
\label{ssec:exp-ablations}

To isolate the contributions that anchoring on historical trials and stepwise factorization used in DeepImagine for training language models, we train two ablation variants of Gemma, both with and without GPT-5-generated reasoning traces.

\textbf{(A1) Direct Configuration-to-Result Prediction.}\\
Gemma is trained to map a configuration $X = (N, A, O)$ and its available trial context $C$ directly to its result $R$. This setting tests whether a language model can directly capture the latent structural mechanism governing clinical trial outcomes. We use $127$k configuration-result pairs to ensure data scale is not the bottleneck.

\textbf{(A2) One-Step Prediction from a Historical Trial Result}\\
Gemma is trained to predict $R^{(T)}$ from a historical configuration and result $\bigl(X^{(1)}, R^{(1)}\bigr)$ together with the target configuration. For each of the $127$k targets, NV-Embed-v2 retrieval followed by GPT-5 verification identifies a similar predecessor, yielding $127$k source-target pairs. A2 learns a direct predictor of the target result. It does not compute the marginal of the trained local operators in {Equation~\ref{eq:marginal_equation}}.

\subsection{Analyses and Discussions}

\paragraph{Seven-anchor DeepImagine improves all GPT backbones.} Table~\ref{tab:main} shows that even with a single anchor ($n=1$), DeepImagine improves over RAG by $4.7$ points for GPT-5, $1.6$ for GPT-5-mini, and $1.9$ for o3-mini. The single-anchor result does not, however, exceed Direct for o3-mini ($72.9$ vs. $73.1$). At seven distinct anchors, DeepImagine exceeds both baselines for all three models tested. RAG itself is below Direct for GPT-5-mini and o3-mini, indicating that retireved historical trials are not always helpful when provided as raw context. DeepImagine, by contrast, consumes the same retrieved trials but processes them through stepwise local counterfactual imagination rather than direct correlation, and the resulting gains are stable across LLMs of different scales and providers.

\begin{wrapfigure}{r}{0.598\linewidth}
  \vspace{-1.2em}
  \centering
  \includegraphics[width=\linewidth]{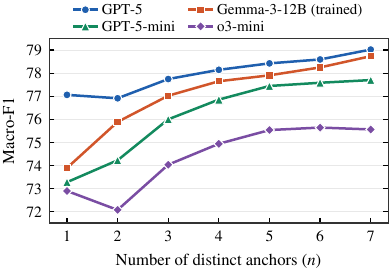}
  \vspace{-1.6em}
  \captionof{figure}{Macro-F1 on the CT Open Summer 2025 as the number of distinct anchors $n$ increases. For $n \ge 2$, each point is the mean over all $\binom{7}{n}$ anchor subsets.}
  \label{fig:anchor-scaling}
  \vspace{-1em}
\end{wrapfigure}

\paragraph{DeepImagine improves on a broader set of prediction baselines.}
Table~\ref{tab:main} B-C reports all ten conventional and task-adapted clinical trial prediction methods. PLR with $\ell_2$ regularization achieves the strongest result at 70.8 macro-F1, followed by MEXA-CTP at 69.8. The trained Gemma system reaches 78.7, significantly outperforming all traditional baseline methods. These comparisons extend beyond a single feed-forward reference model. They evaluate adaptations to our arm--endpoint task, not scores copied from the methods' original trial-level benchmarks; input representations and training procedures differ across methods ({Appendix~\ref{app:baselines}}).

\paragraph{Aggregating across multiple distinct anchors yields substantial additional gains.}
Moving from $n{=}1$ to $n{=}7$ different anchors improves DeepImagine by 1.9 for GPT-5 (to 79.0) and 5.2 for GPT-5-mini (to 78.5). Figure~\ref{fig:anchor-scaling} shows the reported anchor-subset comparison. Though this comparison increases both the amount of historical evidence and the number of pathways, the same-anchor control below compares distinct evidence with repeated generation.

\paragraph{Local-operator training outperforms the one-step Gemma variants.}
Table~\ref{tab:gemma} compares the two one-step training schemes with local-operator training. A1 reaches 54.8/56.9 macro-F1 with/without explanations; adding a historical result in A2 raises these scores to 60.2/63.2. DeepImagine reaches 78.7: 15.5 points above the best A2 setting and 21.8 above the best A1 setting. Explanations alone do not account for the improvement, since both one-step variants score lower with them. The result favors the combined local-supervision and multi-anchor strategy, while the unequal data and inference budgets prevent attributing the entire gap to factorization.

\begin{wraptable}{r}{0.43\linewidth}
  \vspace{-1.2em}
  \centering
  \caption{Gemma training comparisons (macro-F1). ``With/without'' denotes generated explanations. A dash is an unreported setting.}
  \label{tab:gemma}
  \begin{tabular*}{\linewidth}{@{\extracolsep{\fill}}lccc@{}}
    \toprule
    Task & Anchors & With & Without \\
    \midrule
    A1: direct    & 0 & 54.8 & 56.9 \\
    A2: one-step  & 1 & 60.2 & 63.2 \\
    DeepImagine   & 7 & \textbf{78.7} & --- \\
    \bottomrule
  \end{tabular*}
  \vspace{-1em}
\end{wraptable}

\paragraph{Distinct anchors offer more than repeated prediction from one anchor.} At the same count of seven pathways, distinct anchors outperform repeated prediction from one trial by $3.2$, $2.5$, and $3.0$ points for GPT-5, GPT-5-mini, and o3-mini, respectively. Repeating one anchor changes the single-pathway scores by $-1.3$, $2.7$, and $-0.3$ points. Thus, the reported multi-anchor gains is not demonstrated by additional generation from unchanged evidence. The contrast supports the value of different observed starting results rather than repeated use of the same evidence. Its relationship to Proposition~\ref{thm:multi-anchor} is discussed in Section~\ref{ssec:multi-anchor}.

\paragraph{Trained Gemma operators match GPT-5 within DeepImagine.}
The Gemma family of local operators within DeepImagine ($78.7$) matches GPT-5 within the same framework ($79.0$) at $n=7$, despite a much smaller parameter count. This indicates that explicitly teaching a language model to perform local counterfactual imagination from natural and synthetic counterfactuals is comparable in effectiveness to relying on the innate reasoning of a much larger generalist LLM, while being substantially more efficient at deployment.

\subsection{Perturbation Order, Anchor Relevance, and Explanation Quality}
\label{sec:perturbation_order}

The preceding results motivate two further questions: does the order of local changes matter, and how much does prediction depend on the starting observation? The following ablation studies vary these choices separately to investigate the effects.

\begin{wraptable}{r}{0.40\linewidth}
  \vspace{-1em}
  \centering
  \caption{GPT-5 order sensitivity (macro-F1). Anchors, instructions, and judge are fixed. $\Delta$ is relative to $N \rightarrow A \rightarrow O$.}
  \label{tab:order}
  \small
  \begin{tabular}{@{}lcc@{}}
    \toprule
    Update order & Macro-F1 & $\Delta$ \\
    \midrule
    $N \rightarrow A \rightarrow O$ & \textbf{77.06} & --- \\
    \midrule
    $A \rightarrow O \rightarrow N$ & 73.71 & $-3.35$ \\
    $A \rightarrow N \rightarrow O$ & 73.24 & $-3.82$ \\
    $O \rightarrow A \rightarrow N$ & 73.03 & $-4.03$ \\
    $N \rightarrow O \rightarrow A$ & 72.99 & $-4.07$ \\
    $O \rightarrow N \rightarrow A$ & 72.45 & $-4.61$ \\
    \bottomrule
  \end{tabular}
  \vspace{-1em}
\end{wraptable}

\paragraph{The order of local updates matters.}
We use the ordering $N\to A\to O$ because it reduces error propagation and reflects dependencies among the factors. We update $N$ first because enrollment changes primarily affect statistical power while leaving percentage-based results largely unchanged, making this counterfactual the most reliable. We then update $A$ before $O$, since the interpretation of an outcome measure can depend on its associated study arm. This ordering was chosen in consultation with our biomedical expert and used throughout data generation and experiments. We further evaluate all six permutations on the 701 CT Open questions using GPT-5 with DeepImagine; because GPT-5 is not trained on our generated data, this analysis is independent of the generation order. As shown in Table~\ref{tab:order}, the canonical $N\to A\to O$ order achieves the highest macro-F1 of 77.06.

\begin{wraptable}{r}{0.42\linewidth}
  \vspace{-1em}
  \centering
  \caption{Anchor relevance, a separate study. Macro-F1; $\Delta$ is random minus relevant.}
  \label{tab:relevance}
  \small
  \begin{tabular}{@{}lccc@{}}
    \toprule
    Model & Relevant & Random & $\Delta$ \\
    \midrule
    GPT-5 & 75.90 & 71.20 & $-4.70$ \\
    Gemma & 73.75 & 73.09 & $-0.66$ \\
    \bottomrule
  \end{tabular}
  \vspace{-1em}
\end{wraptable}
 
\paragraph{Relevant anchors help, with different sensitivity across models.} 
In another ablation study, we investigate the effect of anchor relevancy to the DeepImagine framework. A relevant anchor is selected by GPT-5 from the top 10 NV-Embed-v2 candidates. Alternatively, we select a random training-corpus anchor. This changes the relevance of the starting configuration to the target configuration. As shown in Table~\ref{tab:relevance}, Random selection lowers GPT-5 from 75.90 to 71.20 and Gemma from 73.75 to 73.09. The smaller drop in Gemma model suggests its lower sensitivity to anchor selection, which could be the result of varied counterfactual settings it has been trained on.

\section{Conclusion and Limitation}
\label{sec:conclusion}

We introduce DeepImagine, a step-wise local counterfactual imagination framework for predicting clinical trial outcomes. On the Summer 2025 split of CT Open, it outperforms traditional machine learning baselines. Applying DeepImagine enables off-the-shelf LLMs to predict trial outcomes more accurately over baselines. Training with DeepImagine also enables the much smaller Gemma-3-12B-it nearly match GPT-5. Limitations include limited coverage of trial factors and reliance on GPT-5 to synthesize intermediate steps when natural counterfactuals are unavailable. Improving the fidelity of these synthetic intermediates, e.g., via verification against future trial readouts, is a natural next step.

\subsection*{AI use statement}

GPT-5 is used for generated explanations, synthetic intermediate results, retrieval relevance evaluation, and prediction as described in Sections~\ref{sec:training} and \ref{sec:experiments}. Generative AI also assisted manuscript editing and figure/table layout.

\subsection*{Ethics statement}

The study uses reported aggregate clinical trial information published in \href{https://clinicaltrials.gov/}{clinicaltrials.gov}. Its predictions and imagined intermediate results are research outputs -- they by no means serve as clinical recommendations or identified patient-level effects.

\subsection*{Reproducibility statement}

Section~\ref{ssec:exp-setup} describes the evaluation setup and predictive strategies. Appendices ~\ref{app:pathway-theory-toc}, \ref{app:training-tuples}, and \ref{app:training-details} provide proofs, supervision details, and training settings. All code are available in the provided GitHub link, which is anonymous for review purpose.

\subsubsection*{Acknowledgments}
This work used Delta (NSF award OAC 2005572) and DeltaAI  (NSF award OAC 2320345) at the National Center for  Supercomputing Applications (NCSA), Anvil AI at Purdue University~\citep{song2022anvil}, and Stampede3 at the Texas Advanced Computing Center (TACC) through allocation CIS261424 from the Advanced Cyberinfrastructure Coordination Ecosystem:  Services \& Support (ACCESS) program~\citep{boerner2023access},  which is supported by U.S. National Science Foundation grants \#2138259, \#2138286, \#2138307, \#2137603, and \#2138296. We gratefully acknowledge Modal for providing additional computational resources through its Academic Grant program.

\newpage

\bibliography{iclr2027_conference}
\bibliographystyle{iclr2027_conference}

\newpage

\noindent\textbf{Appendix: Table of Contents}

\begin{enumerate}[label=\Alph*.]

    \item Formal Definitions and Complete Theoretical Analysis
    \dotfill \pageref{app:pathway-theory-toc}

    \begin{enumerate}[label=\arabic*.]

        \item Notation, Conditioning, and Result Distributions
        \dotfill \pageref{app:notation}

        \item Complete Marginal and Greedy-Pathway Formulation
        \dotfill \pageref{app:pathway-math}

        \item Proof of Proposition: Greedy Approximation of the Marginal
        \dotfill \pageref{app:single-pathway-approx}

        \item Proof of Proposition: Convergence of Aggregating Counterfactual Pathways
        \dotfill \pageref{app:multi-anchor}

        \item Relaxed Convergence under Approximate Calibration
        \dotfill \pageref{app:multi-anchor-relaxed}

    \end{enumerate}

    \item Supervision and Training Details
    \dotfill \pageref{app:supervision}

    \begin{enumerate}[label=\arabic*.]

        \item Observed Labels and Synthetic Intermediate Results
        \dotfill \pageref{app:training-tuples}

        \item Training Details
        \dotfill \pageref{app:training-details}

    \end{enumerate}

    \item Evaluation Protocols and Additional Analyses
    \dotfill \pageref{app:protocol}

    \begin{enumerate}[label=\arabic*.]

        \item Retrieval, Prediction, and Aggregation
        \dotfill \pageref{app:retrieval-aggregation}

        \item Conventional and Task-Adapted Clinical Predictors
        \dotfill \pageref{app:baselines}

        \item Sensitivity Studies and Statistical Reporting
        \dotfill \pageref{app:sensitivity}

        \item Clinical Interpretation and Explanation Assessment
        \dotfill \pageref{app:clinical-interpretation}

    \end{enumerate}

\end{enumerate}

\clearpage

\appendix
\section{Formal Definitions and Complete Theoretical Analysis}
\label{app:pathway-theory-toc}

This appendix develops the full mathematical formulation underlying Section~\ref{sec:deepimagine}. Section~\ref{app:pathway-math} gives the marginal and greedy pathway. Appendices~\ref{app:single-pathway-approx}-\ref{app:multi-anchor-relaxed} give proofs and approximate-calibration extension. We use $T$ states and $T-1$ updates throughout.

\subsection{Notation, Conditioning, and Result Distributions}
\label{app:notation}
Let $(\mathcal X,\mathcal A)$ and $(\mathcal R,\mathcal B)$ be the measurable configuration and result spaces. Superscripts index trials in $R_{j,k}^{(i)}$ and pathway states in $R^{(t)}$; the surrounding definition specifies which is intended. Singletons, contexts, maps, and probability kernels are assumed measurable. For an anchor $C_i$, the full local conditional is
$P_\theta^{\Phi_{t-1}}(\cdot\mid R^{(t-1)},X^{(t-1)},X^{(t)};C_i)$.
The historical context is held fixed along the pathway and suppressed in the equations to preserve the original notation. Across anchors, expectations are over complete records $(X^{(1)},R^{(1)},C)$, not merely the three configuration fields and result. The target and factor order are fixed throughout each application of the propositions.

In the training expression for $\operatorname{concat}(\text{reasoning},R)$, $P_\theta^\Phi$ denotes a complete-response law. In the result-level analysis, it denotes the induced law after a measurable result parser is applied. Different responses can yield the same result. This does not equate the probability of a result with a confidence score stated in an explanation. A failure symbol may make the parsing map total; that mathematical convention does not specify an empirical invalid-output policy.

The modal definition in Definition~\ref{def:prob-parametrization} assumes that a result-level mode and a fixed measurable tie-breaking rule exist. Token-wise greedy decoding need not recover that mode. The greedy proof requires the stated mass bound at the retained result; it does not require exact modal decoding if that bound holds. Probabilities evaluated at a result denote singleton masses, not density heights. Finite result strings or their parsed summaries give a countable result space. More general spaces are allowed, but purely non-atomic conditionals cannot satisfy near-unit singleton dominance.

\subsection{Complete Marginal and Greedy-Pathway Formulation}
\label{app:pathway-math}
The law of total probability gives the following marginal over a fixed pathway. We retain the original integral notation. For a countable result representation, each $dR^{(t)}$ denotes counting measure and each $P_\theta^\Phi$ in the integrand is a mass function. Under a density interpretation, integration uses the corresponding reference measure; the proofs below use probability kernels explicitly.

\begin{equation}
\label{eq:marginal_equation}
\begin{aligned}
P\!\bigl(R^{(T)} \,\big|\, R^{(1)},\, X^{(1)},\, \ldots,\, X^{(T)}\bigr) ={} & \int_{R^{(2)}} \!\!\cdots\! \int_{R^{(T-1)}} \prod_{t=2}^{T-1} P_\theta^{\Phi_{t-1}}\!\bigl(R^{(t)} \,\big|\, R^{(t-1)},\, X^{(t-1)},\, X^{(t)}\bigr) \\
& \times\; P_\theta^{\Phi_{T-1}}\!\bigl(R^{(T)} \,\big|\, R^{(T-1)},\, X^{(T-1)},\, X^{(T)}\bigr) \, dR^{(2)} \cdots dR^{(T-1)}
\end{aligned}
\end{equation}
Our training procedure that uses natural counterfactuals from controlled experiments and synthetic counterfactuals from similar trials (Section~\ref{sec:training}) encourages each local operator $P_\theta^{\Phi_{t-1}}$ to place most of its probability mass on a single value at each step. We formalize this property below as the dominant local conditional assumption, under which the intractable model marginal can be tightly approximated by the probability product along a single greedy pathway which is precisely the quantity DeepImagine computes at inference.

Concretely, the probability of the greedy pathway starting from observed result $R^{(1)}$ is

\begin{equation}
\label{eq:greedy}
\begin{aligned}
&\prod_{t=2}^{T} P_\theta^{\Phi_{t-1}}\!\left(
\hat{R}^{(t)} \mid \hat{R}^{(t-1)}, X^{(t-1)}, X^{(t)}
\right) \\
&= P_\theta^{\Phi_1}\!\left(
\hat{R}^{(2)} \mid R^{(1)}, X^{(1)}, X^{(2)}
\right) \\
&\times
\prod_{t=3}^{T} P_\theta^{\Phi_{t-1}}\!\left(
\hat{R}^{(t)} \mid
f_{\Phi_{t-2}}\!\left(\hat{R}^{(t-2)}, X^{(t-2)}, X^{(t-1)}\right),
X^{(t-1)}, X^{(t)}
\right).
\end{aligned}
\end{equation}
where $\hat{R}^{(t-1)}$ takes the greedy value using the operator $f_{\Phi_{t-2}}(\cdot)$ for $t\geq 3$ and $P^{\Phi_{t-1}}_{\theta}$ is a language model that captures a local counterfactual imagination.

\begin{definition}[\textbf{Dominant Local Conditional}]
\phantomsection
\label{def:dominant}
A local operator $P_\theta^{\Phi_{t-1}}$ is {$\delta_t$-dominant} along the greedy trajectory if its mode $\hat{R}^{(t)} :=  f_{\Phi_{t-1}}(\hat{R}^{(t-1)}, X^{(t-1)}, X^{(t)})$ carries near-unit mass:
\[
P_\theta^{\Phi_{t-1}}\!\bigl(\hat{R}^{(t)} \,\big|\, \hat{R}^{(t-1)},\, X^{(t-1)},\, X^{(t)}\bigr) \;\geq\; 1 - \delta_t,
\]
for some $0 < \delta_t \ll 1$.
\end{definition}

Let $\bigl\{\bigl(X_i^{(1)}, R_i^{(1)}\bigr)\bigr\}_{i \geq 1}$ be an i.i.d.\ sequence of historical trial anchors drawn from the unknown true data-generating law $p_\star$, where $\bigl(X_i^{(1)}, R_i^{(1)}\bigr)$ is the source configuration and its observed result for the $i$-th clinical trial anchor. For any target configuration $X^{(T)}$, running DeepImagine from $\bigl(X_i^{(1)}, R_i^{(1)}\bigr)$ to $X^{(T)}$ yields the marginal distribution
$P\!\left(R^{(T)} \,\middle|\, R_i^{(1)}, X_i^{(1)}, X^{(T)}\right)$
over the target result $R^{(T)}$, see Equation~\ref{eq:marginal_equation}. Let $P_{p_\star}\!\left(R^{(T)} \,\middle|\, X^{(T)}\right)$ denote the true conditional distribution of $R^{(T)}$ given $X^{(T)}$ under the law $p_\star$. 

\subsection{Proof of Proposition~\ref{thm:single-pathway-approx} (Greedy Approximation of the Marginal)}
\label{app:single-pathway-approx}

For convenience we restate the proposition. We write $P_\theta^{\Phi_{t-1}}(\hat{R}^{(t)} \mid \cdots) := P_\theta^{\Phi_{t-1}}(\{\hat{R}^{(t)}\} \mid \cdots)$ for the atom mass that the local operator places at the greedy mode $\hat{R}^{(t)}$. We write $P_\theta$ for the model marginal (denoted $P$ in the main text). Marginalization over intermediate trajectories is performed via integration with respect to the local probability measures.

\begin{restatedproposition}[\textbf{Greedy Approximation of the Marginal}]
Under the $\delta_t$-dominance assumption (Definition~\ref{def:dominant}) for $t = 2, \ldots, T$,
\[
\Bigl| \prod_{t=2}^{T} P_\theta^{\Phi_{t-1}}\!\bigl(\hat{R}^{(t)} \,\big|\, \hat{R}^{(t-1)},\, X^{(t-1)},\, X^{(t)}\bigr) \;-\; P_\theta\!\bigl(\hat{R}^{(T)} \,\big|\, R^{(1)},\, X^{(1)},\, \ldots,\, X^{(T)}\bigr) \Bigr| \;\leq\; \sum_{t=2}^{T} \delta_t.
\]
\end{restatedproposition}
The same assumptions also imply the total-variation consequence proved below:
\[
\left\|P_\theta(\cdot\mid R^{(1)},X^{(1)},\ldots,X^{(T)})-\boldsymbol\delta_{\hat R^{(T)}}\right\|_{\mathrm{TV}}
\leq\sum_{t=2}^{T}\delta_t.
\]

\begin{proof}
 Throughout, all conditioning is on the fixed inputs $(R^{(1)}, X^{(1)}, \ldots, X^{(T)})$, which we abbreviate as $(\cdots)$.

By the probabilistic parametrization (Definition~\ref{def:prob-parametrization}) and the chain-rule decomposition (Definition~\ref{def:stepwise-cf}), the joint probability measure over the trajectory $(R^{(2)}, \ldots, R^{(T)}) \in \mathcal{R}^{T-1}$ factorizes as
\[
P_\theta\!\bigl(dR^{(2)}, \ldots, dR^{(T)} \,\big|\, \cdots\bigr) \;=\; \prod_{t=2}^{T} P_\theta^{\Phi_{t-1}}\!\bigl(dR^{(t)} \,\big|\, R^{(t-1)}, X^{(t-1)}, X^{(t)}\bigr).
\]
Restricting both sides to the singleton $\{(\hat{R}^{(2)}, \ldots, \hat{R}^{(T)})\} \subseteq \mathcal{R}^{T-1}$ gives the joint atom mass:
\begin{equation}
\label{eq:joint-equals-product}
P_\theta\!\bigl(\hat{R}^{(2)}, \ldots, \hat{R}^{(T)} \,\big|\, \cdots\bigr) \;=\; \prod_{t=2}^{T} P_\theta^{\Phi_{t-1}}\!\bigl(\hat{R}^{(t)} \,\big|\, \hat{R}^{(t-1)}, X^{(t-1)}, X^{(t)}\bigr).
\end{equation}

By $\delta_t$-dominance, $P_\theta^{\Phi_{t-1}}(\hat{R}^{(t)} \mid \hat{R}^{(t-1)}, X^{(t-1)}, X^{(t)}) \geq 1 - \delta_t$ for every $t \in \{2, \ldots, T\}$. Multiplying these inequalities,
\begin{equation}
\label{eq:joint-lower-bound}
\prod_{t=2}^{T} P_\theta^{\Phi_{t-1}}\!\bigl(\hat{R}^{(t)} \,\big|\, \hat{R}^{(t-1)}, \cdots\bigr) \;\geq\; \prod_{t=2}^{T}(1 - \delta_t) \;\geq\; 1 - \sum_{t=2}^{T}\delta_t,
\end{equation}
where the final inequality is the Weierstrass product inequality, proved by induction on $T$. The base case $T = 2$ gives $1 - \delta_2 \geq 1 - \delta_2$, which is trivial. For the inductive step, assuming $\prod_{t=2}^{T}(1-\delta_t) \geq 1 - \sum_{t=2}^{T}\delta_t$,
\[
\prod_{t=2}^{T+1}(1-\delta_t) \;=\; (1-\delta_{T+1})\prod_{t=2}^{T}(1-\delta_t) \;\geq\; (1-\delta_{T+1})\Bigl(1 - \sum_{t=2}^{T}\delta_t\Bigr).
\]
Expanding the right-hand side,
\[
(1-\delta_{T+1})\Bigl(1 - \sum_{t=2}^{T}\delta_t\Bigr) \;=\; 1 - \sum_{t=2}^{T+1}\delta_t \,+\, \delta_{T+1}\sum_{t=2}^{T}\delta_t \;\geq\; 1 - \sum_{t=2}^{T+1}\delta_t,
\]
since $\delta_{T+1}\sum_{t=2}^{T}\delta_t \geq 0$ for $\delta_t \in [0,1]$.

The marginal atom mass of $R^{(T)}$ at the greedy value $\hat{R}^{(T)}$ is obtained by integrating over the intermediate result space: \[
\begin{aligned}
P_\theta\!\bigl(\hat{R}^{(T)} \,\big|\, \cdots\bigr)
&= \int_{\mathcal{R}^{T-2}}\!\!
P_\theta^{\Phi_{T-1}}\!\bigl(\hat{R}^{(T)} \,\big|\,
R^{(T-1)}, X^{(T-1)}, X^{(T)}\bigr) \\
&\qquad {}\times
\prod_{t=2}^{T-1}
P_\theta^{\Phi_{t-1}}\!\bigl(dR^{(t)} \,\big|\,
R^{(t-1)}, X^{(t-1)}, X^{(t)}\bigr).
\end{aligned}
\]
The greedy trajectory $(\hat{R}^{(2)}, \ldots, \hat{R}^{(T-1)}, \hat{R}^{(T)})$ contributes the joint atom mass $P_\theta(\hat{R}^{(2)}, \ldots, \hat{R}^{(T)} \mid \cdots)$ to the integral above. Since each local operator has an atom of mass $\geq 1 - \delta_t$ at $\hat{R}^{(t)}$, this contribution is non-negative, and the integral is at least this single contribution:
\[
P_\theta\!\bigl(\hat{R}^{(T)} \,\big|\, \cdots\bigr) \;\geq\; P_\theta\!\bigl(\hat{R}^{(2)}, \ldots, \hat{R}^{(T)} \,\big|\, \cdots\bigr).
\]
Combining with \eqref{eq:joint-equals-product} and \eqref{eq:joint-lower-bound},
\begin{equation}
\label{eq:marginal-lower-bound}
\begin{aligned}
1 \;\geq\; P_\theta\!\bigl(\hat{R}^{(T)} \,\big|\, \cdots\bigr)
&\;\geq\; P_\theta\!\bigl(\hat{R}^{(2)}, \ldots, \hat{R}^{(T)} \,\big|\, \cdots\bigr) \\
&\;=\; \prod_{t=2}^{T} P_\theta^{\Phi_{t-1}}\!\bigl(\hat{R}^{(t)} \,\big|\, \hat{R}^{(t-1)}, \cdots\bigr)
\;\geq\; 1 - \sum_{t=2}^{T}\delta_t.
\end{aligned}
\end{equation}
Both the stepwise product and the marginal atom mass at $\hat{R}^{(T)}$ thus lie in the interval $\bigl[1 - \sum_t \delta_t,\, 1\bigr]$, with the product no larger than the marginal. This yields the first claim:
\[
0 \;\leq\; P_\theta\!\bigl(\hat{R}^{(T)} \,\big|\, \cdots\bigr) \,-\, \prod_{t=2}^{T} P_\theta^{\Phi_{t-1}}\!\bigl(\hat{R}^{(t)} \,\big|\, \hat{R}^{(t-1)}, \cdots\bigr) \;\leq\; 1 - \Bigl(1 - \sum_{t=2}^{T}\delta_t\Bigr) \;=\; \sum_{t=2}^{T}\delta_t.
\]

For the total variation corollary, recall that for any probability measure $P$ on $(\mathcal{R}, \mathcal{B})$ and any point $r^\star \in \mathcal{R}$,
\[
\bigl\| P - \boldsymbol{\delta}_{r^\star} \bigr\|_{\mathrm{TV}} \;=\; \sup_{B \in \mathcal{B}}\, \bigl| P(B) - \boldsymbol{\delta}_{r^\star}(B) \bigr| \;=\; 1 - P(\{r^\star\}),
\]
where the supremum is attained at $B = \{r^\star\}^c$ (or equivalently $B = \{r^\star\}$). Applying this with $r^\star = \hat{R}^{(T)}$ and \eqref{eq:marginal-lower-bound},
\[
\Bigl\| P_\theta\!\bigl(\,\cdot\, \,\big|\, R^{(1)}, X^{(1)}, \ldots, X^{(T)}\bigr) \,-\, \boldsymbol{\delta}_{\hat{R}^{(T)}} \Bigr\|_{\mathrm{TV}} \;=\; 1 - P_\theta\!\bigl(\hat{R}^{(T)} \,\big|\, \cdots\bigr) \;\leq\; \sum_{t=2}^{T}\delta_t.
\]
This establishes the corollary and completes the proof.
\end{proof}

\paragraph{Remark (Linear-in-$T$ scaling).}
The bound grows linearly in the pathway length $T$, not exponentially. The sum accumulates per-step concentration gaps; it is informative when those gaps are small, rather than a bound on clinical prediction error. For the three-factor setting $T=4$, the bound reduces to $\delta_2+\delta_3+\delta_4$, recovering the bound stated in the main text statement of Proposition~\ref{thm:single-pathway-approx}.

\subsection{Proof of Proposition~\ref{thm:multi-anchor} (Convergence of Aggregating Counterfactual Pathways)}
\label{app:multi-anchor}

For convenience we restate the proposition. Throughout this appendix, the result space $(\mathcal R,\mathcal B)$ follows Appendix~\ref{app:notation}; it need not be non-atomic, and the configuration space $\mathcal{X}$ is equipped with a $\sigma$-algebra $\mathcal{A}$. All single-pathway predictions $P_\theta(\,\cdot\, \mid R^{(1)}, X^{(1)}, X^{(T)})$ are interpreted as probability measures on $(\mathcal{R}, \mathcal{B})$.

Let $H_i=(X_i^{(1)},R_i^{(1)},C_i)$ be i.i.d. complete historical records under $p_\star$. In the original pair notation below, $(X_i^{(1)},R_i^{(1)})$ abbreviates the anchor together with its associated fixed context: the expectations integrate over $H_i$, and the model conditionals retain $C_i$. This convention does not assume that a stripped configuration--result pair determines its context. For each fixed target configuration $X^{(T)}$, define the \textit{prior predictive distribution}, the anchor-marginal of DeepImagine's single-pathway prediction:
\[
M_\theta\!\bigl(B \,\big|\, X^{(T)}\bigr) \;:=\; \mathbb{E}_{(X^{(1)}, R^{(1)}) \sim p_\star}\!\Bigl[ P_\theta\!\bigl(B \,\big|\, R^{(1)}, X^{(1)}, X^{(T)}\bigr) \Bigr], \qquad B \in \mathcal{B}.
\]

\begin{definition}[\textbf{Calibration}]
\label{def:calibration}
The model $P_\theta$ is \textit{calibrated} at the target configuration $X^{(T)}$ if its prior predictive coincides with the true target conditional as probability measures: for every fixed measurable event $B\in\mathcal B$,
\[
M_\theta\!\bigl(B \,\big|\, X^{(T)}\bigr) \;=\; P_{p_\star}\!\bigl(B \,\big|\, X^{(T)}\bigr).
\]
\end{definition}

\begin{restatedproposition}[Convergence of Aggregating Counterfactual Pathways]
Suppose $P_\theta$ is calibrated at the target configuration $X^{(T)}$ in the sense of Definition~\ref{def:calibration}. Then, for every measurable event $B \in \mathcal{B}$,
\[
\frac{1}{n} \sum_{i=1}^{n} P_\theta\!\bigl(B \,\big|\, R_i^{(1)}, X_i^{(1)}, X^{(T)}\bigr) \;\longrightarrow\; P_{p_\star}\!\bigl(B \,\big|\, X^{(T)}\bigr) \qquad \text{a.s.\ as } n \to \infty.
\]
\end{restatedproposition}

\begin{proof}
Fix a measurable event $B \in \mathcal{B}$ and a target configuration $X^{(T)}$. Define
\[
Y_i \;:=\; P_\theta\!\bigl(B \,\big|\, R_i^{(1)}, X_i^{(1)}, X^{(T)}\bigr), \qquad i = 1, 2, \ldots.
\]
Each $Y_i$ is a probability, so $0 \leq Y_i \leq 1$. Because $Y_i$ is a fixed measurable function of the complete record $H_i$ and those records are i.i.d., the random variables $Y_1, Y_2, \ldots$ are also i.i.d.\ and uniformly bounded, hence integrable.

By Kolmogorov's strong law of large numbers,
\[
\frac{1}{n} \sum_{i=1}^{n} Y_i \;\longrightarrow\; \mathbb{E}[Y_1] \qquad \text{a.s.\ as } n \to \infty.
\]

By definition of the prior predictive,
\[
\mathbb{E}[Y_1] \;=\; \mathbb{E}_{(X_1^{(1)}, R_1^{(1)}) \sim p_\star}\!\Bigl[ P_\theta\!\bigl(B \,\big|\, R_1^{(1)}, X_1^{(1)}, X^{(T)}\bigr) \Bigr] \;=\; M_\theta(B \mid X^{(T)}).
\]

By the calibration assumption (Definition~\ref{def:calibration}),
\[
M_\theta(B \mid X^{(T)}) \;=\; P_{p_\star}(B \mid X^{(T)}).
\]

\[
\frac{1}{n} \sum_{i=1}^{n} P_\theta\!\bigl(B \,\big|\, R_i^{(1)}, X_i^{(1)}, X^{(T)}\bigr) \;\longrightarrow\; P_{p_\star}\!\bigl(B \,\big|\, X^{(T)}\bigr) \qquad \text{a.s.}
\]
Since $B \in \mathcal{B}$ was arbitrary, this completes the proof.
\end{proof}

This is almost-sure convergence for each fixed event. The probability-one set may depend on that event; the statement alone is not uniform convergence over all measurable events. The next subsection supplies conditions for total-variation convergence.

\subsection{Relaxed Convergence under Approximate Calibration}
\label{app:multi-anchor-relaxed}

We state and prove a relaxed version of Proposition~\ref{thm:multi-anchor} that allows the calibration assumption to hold only approximately. This relaxation is important in practice: averaging independent predictions can reduce sampling variance without eliminating their systematic calibration discrepancy. We adopt the same setup and notation as Appendix~\ref{app:multi-anchor}, including the prior predictive
\[
M_\theta\!\bigl(\,\cdot\, \,\big|\, X^{(T)}\bigr) \;:=\; \mathbb{E}_{(X^{(1)}, R^{(1)}) \sim p_\star}\!\Bigl[ P_\theta\!\bigl(\,\cdot\, \,\big|\, R^{(1)}, X^{(1)}, X^{(T)}\bigr) \Bigr],
\]
and the aggregated prediction $\bar{P}_n(\,\cdot\, \mid X^{(T)}) := \tfrac{1}{n}\sum_{i=1}^n P_\theta(\,\cdot\, \mid R_i^{(1)}, X_i^{(1)}, X^{(T)})$.

We assume that, for $p_\star$-almost every anchor pair $(X^{(1)}, R^{(1)})$, the predictive measure $P_\theta(\,\cdot\, \mid R^{(1)}, X^{(1)}, X^{(T)})$ admits a jointly measurable density $p_\theta(r\mid R^{(1)},X^{(1)},X^{(T)})$ with respect to a common $\sigma$-finite reference measure $\nu$ on $(\mathcal R,\mathcal B)$ (e.g., counting measure for countable result strings, or Lebesgue measure when applicable). Under this assumption, the prior predictive admits density $m_\theta(r \mid X^{(T)}) := \mathbb{E}[p_\theta(r \mid R^{(1)}, X^{(1)}, X^{(T)})]$, and the aggregated predictive admits density
\[
\bar{p}_n(r \mid X^{(T)}) \;=\; \frac{1}{n} \sum_{i=1}^n p_\theta\!\bigl(r \,\big|\, R_i^{(1)}, X_i^{(1)}, X^{(T)}\bigr).
\]

\begin{definition}[\textbf{$\epsilon$-Approximate Calibration}]
\label{def:approx-calibration}
The model $P_\theta$ is \textit{$\epsilon$-approximately calibrated} at the target configuration $X^{(T)}$ if
\[
\bigl\| M_\theta\!\bigl(\,\cdot\, \,\big|\, X^{(T)}\bigr) \;-\; P_{p_\star}\!\bigl(\,\cdot\, \,\big|\, X^{(T)}\bigr) \bigr\|_{\mathrm{TV}} \;\leq\; \epsilon,
\]
for some $\epsilon \geq 0$. Strict calibration (Definition~\ref{def:calibration}) corresponds to $\epsilon = 0$.
\end{definition}

\begin{proposition}[\textbf{Relaxed Convergence of Aggregating Counterfactual Pathways}]
\label{thm:multi-anchor-relaxed}
Suppose $P_\theta$ is $\epsilon$-approximately calibrated at $X^{(T)}$ (Definition~\ref{def:approx-calibration}), and the predictive measures admit densities w.r.t.\ a common reference measure $\nu$. Then
\[
\limsup_{n \to \infty} \;\bigl\| \bar{P}_n\!\bigl(\,\cdot\, \,\big|\, X^{(T)}\bigr) \;-\; P_{p_\star}\!\bigl(\,\cdot\, \,\big|\, X^{(T)}\bigr) \bigr\|_{\mathrm{TV}} \;\leq\; \epsilon \qquad \text{a.s.}
\]
With $\epsilon=0$, the additional density assumptions strengthen the eventwise conclusion of Proposition~\ref{thm:multi-anchor} to TV convergence.
\end{proposition}

\begin{proof}
For each fixed $r \in \mathcal{R}$ at which $m_\theta(r \mid X^{(T)}) < \infty$, the random variables $\bigl\{p_\theta(r \mid R_i^{(1)}, X_i^{(1)}, X^{(T)})\bigr\}_{i \geq 1}$ are i.i.d., non-negative, with mean
\[
\mathbb{E}\bigl[p_\theta(r \mid R_1^{(1)}, X_1^{(1)}, X^{(T)})\bigr] \;=\; m_\theta(r \mid X^{(T)}) \;<\; \infty.
\]
By Kolmogorov's strong law of large numbers,
\[
\bar{p}_n(r \mid X^{(T)}) \;\longrightarrow\; m_\theta(r \mid X^{(T)}) \qquad \text{a.s.}
\]
Since $m_\theta(\,\cdot\, \mid X^{(T)})$ is a probability density, it is finite $\nu$-a.e., so this convergence holds a.s.\ for $\nu$-a.e.\ $r$.

Define the joint complement set
\[
E^c \;:=\; \bigl\{(\omega, r) \in \Omega \times \mathcal{R} \,:\, \bar{p}_n(r \mid X^{(T)})(\omega) \not\to m_\theta(r \mid X^{(T)}) \text{ as } n \to \infty\bigr\},
\]
which is jointly measurable. For $\nu$-a.e.\ fixed $r$, $\mathbb{P}\bigl(\{\omega : (\omega, r) \in E^c\}\bigr) = 0$. By Fubini's theorem,
\[
0 \;=\; \int_{\mathcal{R}} \mathbb{P}\bigl(\{\omega : (\omega, r) \in E^c\}\bigr) \, \nu(dr) \;=\; \int_{\Omega} \nu\bigl(\{r : (\omega, r) \in E^c\}\bigr) \, d\mathbb{P}(\omega).
\]
Since the integrand on the right is non-negative, $\nu\bigl(\{r : (\omega, r) \in E^c\}\bigr) = 0$ for $\mathbb{P}$-a.e.\ $\omega$. That is, there exists a probability-one event $\Omega_\infty$ on which the convergence $\bar{p}_n(r \mid X^{(T)}) \to m_\theta(r \mid X^{(T)})$ holds simultaneously for $\nu$-a.e.\ $r$.

On the event $\Omega_\infty$, we have:
(i) $\bar{p}_n(\,\cdot\, \mid X^{(T)})$ and $m_\theta(\,\cdot\, \mid X^{(T)})$ are non-negative probability densities w.r.t.\ $\nu$, both integrating to one; and
(ii) $\bar{p}_n(\,\cdot\, \mid X^{(T)}) \to m_\theta(\,\cdot\, \mid X^{(T)})$ pointwise $\nu$-a.e.
Scheff\'e's lemma asserts that pointwise a.e.\ convergence of probability densities implies $L^1$ convergence:
\[
\int_{\mathcal{R}} \bigl| \bar{p}_n(r \mid X^{(T)}) - m_\theta(r \mid X^{(T)}) \bigr| \, \nu(dr) \;\longrightarrow\; 0 \qquad \text{on } \Omega_\infty.
\]
Since the $L^1$ distance between densities equals twice the total variation distance between the corresponding probability measures,
\[
\bigl\| \bar{P}_n\!\bigl(\,\cdot\, \,\big|\, X^{(T)}\bigr) \;-\; M_\theta\!\bigl(\,\cdot\, \,\big|\, X^{(T)}\bigr) \bigr\|_{\mathrm{TV}} \;\longrightarrow\; 0 \qquad \text{a.s.}
\]

By the triangle inequality for total variation,
\[
\bigl\| \bar{P}_n - P_{p_\star} \bigr\|_{\mathrm{TV}} \;\leq\; \underbrace{\bigl\| \bar{P}_n - M_\theta \bigr\|_{\mathrm{TV}}}_{\to\, 0 \text{ a.s.\ }} \;+\; \underbrace{\bigl\| M_\theta - P_{p_\star} \bigr\|_{\mathrm{TV}}}_{\leq\, \epsilon \text{ by Def.~\ref{def:approx-calibration}}},
\]
where all measures are conditioned on $X^{(T)}$. Taking $\limsup$ as $n \to \infty$,
\[
\limsup_{n \to \infty} \bigl\| \bar{P}_n(\,\cdot\, \mid X^{(T)}) - P_{p_\star}(\,\cdot\, \mid X^{(T)}) \bigr\|_{\mathrm{TV}} \;\leq\; 0 + \epsilon \;=\; \epsilon \qquad \text{a.s.}
\]
This completes the proof.
\end{proof}

\section{Supervision and Training Details}
\label{app:supervision}
\subsection{Observed Labels and Synthetic Intermediate Results}
\label{app:training-tuples}

For a natural outcome pair from arm $k$ in trial $i$, the first observed result is an input and the second is the target. GPT-5 supplies the connecting explanation, and the operator is trained to generate the explanation followed by the observed target result:
\begin{equation}\label{eq:natural-supervision}
P_\theta^O\!\left(\textrm{concat }(\textrm{reasoning, }R_{j_2,k}^{(i)} )\,\middle|\, R_{j_1,k}^{(i)},\, X^{(i,j_1,k)},\, X^{(i,j_2,k)}\right).
\end{equation}

The corresponding arm-update objective exchanges the arm index while retaining the outcome measure. These response targets are fitted with standard cross-entropy supervised fine-tuning. Observed aggregate comparisons do not supply both potential outcomes for any individual patient.

For two retrieved historical trials $C_1,C_2$, superscripts $(1),(2)$ identify their observed configurations and results; $(a),(b),(c)$ identify teacher-generated results. The three-factor synthetic bridge is
\begin{equation}\label{eq:synthetic-bridge}
\begin{aligned}
&(N^{(1)}, A^{(1)}, O^{(1)}, R^{(1)}) \to (N^{(2)}, A^{(1)}, O^{(1)}, R^{(a)}) \\
&\to (N^{(2)}, A^{(2)}, O^{(1)}, R^{(b)}) \to (N^{(2)}, A^{(2)}, O^{(2)}, R^{(c)}),
\end{aligned}
\end{equation}
The historical context $C_1$ is fixed and suppressed in these tuples. The teacher receives both observed endpoint results $R^{(1)},R^{(2)}$, but $R^{(a)},R^{(b)},R^{(c)}$ remain generated. In particular, $R^{(2)}$ is reference evidence, not automatically the observed label for the last imagined configuration under $C_1$. Matching $N,A,O$ does not align every contextual attribute. The first synthetic tuple has an observed input result; later tuples have generated input results. This is why the term synthetic does not imply that every endpoint of every tuple is generated.

\begin{table}[htbp]
\centering\small

\caption{Three-factor supervision counts before the 10\% tuning holdout. Counts are training tuples, not independent trials. All accompanying explanations are generated by GPT-5.}
\label{tab:training-data}
\begin{tabular}{lrrr}
\toprule
Operator & Natural & Synthetic & Total\\
\midrule
Enrollment $N$ & --- & 45k & 45k\\
Study arm $A$ & 57k & 45k & 102k\\
Outcome measure $O$ & 82k & 45k & 127k\\
\midrule
Total & 139k & 135k & 274k\\
\bottomrule
\end{tabular}

\end{table}

\subsection{Training Details}
\label{app:training-details}

We use GPT-5 \citep{openai2025gpt5}, with medium reasoning effort and low verbosity, to generate synthetic counterfactual reasoning trajectories for the three endogenous factors: enrollment count $N$, study arm $A$, and outcome measure $O$, yielding 45k synthetic training tuples per factor (Section~\ref{sec:training}). For study-arm and outcome-measure perturbations, these synthetic tuples augment 57k and 82k natural counterfactual pairs, respectively; for enrollment-count perturbations, no natural counterfactuals are available, so training relies solely on synthetic data.

 We fine-tuned a separate Gemma-3-12B-it \citep{gemma3report} model for each of the three local operators $P_\theta^{\Phi}$, $\Phi \in \{N, A, O\}$, using the identical hyperparameter configuration listed in Table~\ref{tab:hyperparameters}.The same configuration is also used for the two ablation training schemes in Section~\ref{ssec:exp-ablations} (A1: direct configuration-to-result prediction; A2: one-step prediction from a historical trial result), in both their with-reasoning and without-reasoning variants. LoRA adapters were applied to all linear layers with rank $r=32$, scaling factor $\alpha=64$, and dropout $0.05$. We trained for one epoch with a per-device batch size of 1 and gradient accumulation of 2. Optimization used AdamW with a peak learning rate of $5\times 10^{-5}$, cosine decay to a minimum of $10\%$ of the peak learning rate, $5\%$ warm-up steps, zero weight decay, and gradient clipping at $1.0$. The maximum sequence length was 8{,}192 tokens. Training was performed on a single node with 7$\times$80GB H100 GPUs using FlashAttention-4.

\begin{table}[htbp]
  \centering
  \small
  \caption{Hyperparameters used uniformly across all Gemma-based experiments, including the three local operators $P_\theta^{\Phi}$ for $\Phi \in \{N, A, O\}$ and both ablation training schemes (A1, A2).}  
  \label{tab:hyperparameters}
  \begin{tabularx}{\linewidth}{
  @{}
  l
  >{\raggedright\arraybackslash}X
  @{\hspace{1.5em}}
  l
  >{\raggedright\arraybackslash}X
  @{}
}
    \toprule
    \textbf{Hyperparameter} & \textbf{Value} & \textbf{Hyperparameter} & \textbf{Value} \\
    \midrule
    \multicolumn{2}{@{}l}{\textit{Model and fine-tuning}} & \multicolumn{2}{l}{\textit{Hardware and precision}} \\
    \quad Base model                     & Gemma-3-12B-it            & \quad Compute   & 7$\times$GPU, single-node DDP \\
    \quad Fine-tuning method             & LoRA (all linear layers)  & \quad Attention & FlashAttention-4 \\
    \quad LoRA rank / $\alpha$ / dropout & $32$\,/\,$64$\,/\,$0.05$  &                 &  \\
    \addlinespace[0.5em]
    \multicolumn{2}{@{}l}{\textit{Optimization}} & \multicolumn{2}{l}{\textit{Batching and schedule}} \\
    \quad Optimizer              & AdamW              & \quad Epochs               & $1$      \\
    \quad Maximum learning rate  & $5 \times 10^{-5}$ & \quad Per-GPU batch size   & $1$      \\
    \quad Minimum learning rate  & $5 \times 10^{-6}$ & \quad Gradient accumulation & $2$     \\
    \quad LR scheduler           & Cosine decay       & \quad Effective batch size & $14$     \\
    \quad Warmup ratio           & $5\%$              & \quad Max sequence length  & $8{,}192$ \\
    \quad Weight decay           & $0.0$              &                            &          \\
    \quad Gradient clipping      & $1.0$              &                            &          \\
    \bottomrule
  \end{tabularx}
\end{table}

\section{Evaluation Protocols and Additional Analyses}
\label{app:protocol}
\subsection{Retrieval, Prediction, and Aggregation}
\label{app:retrieval-aggregation}

The evaluation cohort used here comprises 701 Summer 2025 CT Open questions: 442 positives and 259 negatives, with a result-availability cutoff of 2025-08-31 \citep{wang2026ctopen}. This is the cohort specified in our experiments, not a claim that every public CT Open release contains exactly 701 questions. The API snapshots are GPT-5-2025-08-07, GPT-5-mini-2025-08-07, and o3-mini-2025-01-31. Each method predicts one binary answer per question; macro-F1 averages the two class-specific F1 scores.

The retrieval representation uses the available registry description, excluding results. NV-Embed-v2 \citep{lee2025nvembed} supplies text embeddings and cosine similarity; the retrieval protocol uses language-model relevance verification. Selected historical results are then available to the prediction stage, whereas target results are withheld. Direct receives only the target question and description. RAG adds up to seven historical trials and observed results per target arm. The corresponding DeepImagine comparison uses the same historical evidence to initialize local pathways.

For each target arm, the three-factor pathway replaces enrollment, then study arm, then outcome measure. Each operator conditions on the current and next configurations, the preceding result, and fixed historical context. Its response contains an explanation and updated result; the parsed result is passed to the next operator. Terminal predictions are averaged on the target endpoint's numerical scale before the final binary comparison. This is not a vote over intermediate answers. For API-model comparisons, the generating backbone performs the final comparison. Forecasting that a trial will meet a significance criterion is not itself a significance test on target observations.

For two arms, $n$ anchors per arm, and $T-1$ updates, the local stage uses $2n(T-1)$ operator calls. At $T=4$, this is $6n$ calls, excluding retrieval verification and final comparison. This is a structural call count, not a measurement of tokens or latency. The same-anchor control matches pathway counts, not token lengths or cost. The Gemma strategy comparison likewise uses the same fine-tuning hyperparameters but not the same number of tuples: A1 and A2 each use 127k pairs, whereas the three operators use 274k tuples in total before the tuning holdout.

\subsection{Conventional and Task-Adapted Clinical Predictors}
\label{app:baselines}

\paragraph{Shared features and preprocessing.}
The six conventional baselines use available numerical, categorical, and textual trial features, including trial descriptions, study arms, and outcome measures. Numerical features are used directly, with column-median imputation for missing values; the $k$NN+RF baseline instead uses nearest-neighbor imputation \citep{lo2019machine}. Categorical features are one-hot encoded, with missing entries replaced by the most frequent category. Text fields are embedded with NV-Embed-v2 \citep{lee2025nvembed}; missing fields are represented by zero vectors. PLR uses the full embeddings, whereas the other five models reduce text representations to 32 dimensions using PCA. PLR therefore differs in both classifier and representation dimension; its advantage should not be attributed to the classifier alone.

\paragraph{Feed-forward network.}
The FFNN maps the concatenated trial features directly to the prediction label. It has two hidden layers of sizes 512 and 128 with ReLU activations, regularization coefficient $10^{-4}$, initial learning rate $10^{-3}$, and at most 200 iterations. The checkpoint with the best validation performance is used for evaluation. Unlike local-operator training, this model does not receive supervision on intermediate result transitions. Table~\ref{tab:mlhyperparams} preserves the full supplied conventional-model settings.

\paragraph{Clinical prediction models.}
HINT \citep{fu2022hint} combines representations of drug molecules, diseases, and eligibility criteria. In the task adaptation, drug SMILES, disease ICD codes, and eligibility descriptions retain their original roles, while outcome descriptions and the compared arms supply the information needed for an arm--endpoint question. SPOT \citep{wang2023spot}, LIFTED \citep{zheng2025multimodalclinicaltrialoutcome}, and MEXA-CTP \citep{zhang2025mexactpmodeexpertscrossattention} are likewise adapted to the present prediction target while retaining their core architectures. Hyperparameters follow the original implementations, and checkpoints are selected by validation performance. These scores evaluate the supplied task adaptations; they are not the methods' published performance on their original trial-level tasks.

\begin{table}[ht]
\centering\small

\caption{Conventional-model hyperparameters supplied with the updated experiments. Triplets denote Endpoint / Superiority / Comparative Effect settings; a single value applies to all three. The extra task settings are retained for reproducibility of the source protocol, not presented as additional evaluations in this paper.}
\label{tab:mlhyperparams}
\setlength{\tabcolsep}{8pt}\renewcommand{\arraystretch}{1.06}
\begin{tabular}{@{}llc@{}}
\toprule
Model & Hyperparameter & Setting\\
\midrule
\multirow{4}{*}{PLR \citep{lo2019machine}}
 & Penalty & $\ell_2$\\
 & Inverse regularization $C$ & 1.0\\
 & Solver & \texttt{lbfgs}\\
 & Maximum iterations & 3000\\
\midrule
\multirow{3}{*}{Random forest \citep{lo2019machine}}
 & Estimators & 300 / 200 / 200\\
 & Maximum depth & None\\
 & Minimum samples per leaf & 6 / 4 / 4\\
\midrule
\multirow{5}{*}{$k$NN + random forest \citep{lo2019machine}}
 & Imputation neighbors & 5\\
 & Neighbor weights & Uniform\\
 & Estimators & 200 / 100 / 100\\
 & Maximum depth & None\\
 & Minimum samples per leaf & 4\\
\midrule
\multirow{5}{*}{FFNN \citep{tranchevent2019deep}}
 & Hidden layers & (512, 128)\\
 & Activation & ReLU\\
 & Regularization $\alpha$ & $10^{-4}$\\
 & Initial learning rate & $10^{-3}$\\
 & Maximum iterations & 200\\
\midrule
\multirow{4}{*}{AdaBoost \citep{Fan2020KPCA}}
 & Estimators & 300 / 300 / 200\\
 & Learning rate & 1.0\\
 & Base-tree maximum depth & 2\\
 & Base-tree minimum samples per leaf & 2\\
\midrule
\multirow{3}{*}{XGBoost \citep{SIAH2021100312}}
 & Estimators & 300\\
 & Learning rate & 0.1 / 1.0 / 1.0\\
 & Maximum depth & 12\\
\bottomrule
\end{tabular}

\end{table}
\FloatBarrier

\subsection{Sensitivity Studies and Statistical Reporting}
\label{app:sensitivity}

The order experiment evaluates every permutation of $N,A,O$ with GPT-5 on 701 questions, holding anchors, instructions, and the judge fixed within that study. The anchor-relevance experiment uses the canonical $N\rightarrow A\rightarrow O$ order and compares a relevant historical anchor with a random training-corpus anchor. Relevant selection uses a top-10 NV-Embed-v2 shortlist followed by GPT-5 selection. The studies retain their distinct controls; neither is substituted for a seven-anchor measurement.

The reported bootstrap intervals quantify uncertainty in individual macro-F1 estimates, rather than standard deviations or uncertainty in paired differences between methods. Therefore, interval overlap or separation alone does not establish statistical significance. Uncertainty estimates for the Gemma ablations, conventional baselines, and sensitivity studies remain unreported. Training-seed variability, paired uncertainty, and measured token or latency costs also remain unreported. In particular, the 3.35--4.61-point differences in the perturbation-order study are descriptive and do not establish that all pairwise differences are statistically significant.

\subsection{Clinical Interpretation and Explanation Assessment}
\label{app:clinical-interpretation}

DeepImagine uses counterfactual imagination procedurally, not as a claim that each update is an identified causal intervention \citep{pearl2009causality}. The three factors have different roles. Enrollment concerns sampling design: with population and intervention fixed, sample size need not change the population effect, but can affect estimate precision and the chance of meeting a significance criterion. The arm update concerns an intervention contrast. The outcome update changes the measured quantity, not the assigned treatment.

Natural arm pairs are observed aggregate comparisons from a controlled trial, not both potential outcomes for individual participants. Exact single-factor matching additionally requires agreement of enrollment and the other encoded attributes; randomization alone does not guarantee equal arm sizes. Natural outcome pairs may differ in assessment time, missingness, and analysis population. The reported extraction counts do not themselves establish that every pair satisfies strict matching. Retrieval identifies related examples but does not establish transportability across patient populations or remove contextual differences.

A biomedical expert assessed 90 generated explanations from the source supervision corpus: 87 were clinically sound, two potentially problematic, and one uncertain. The assessment was not broken down by factor or by natural versus synthetic supervision. It checks explanation plausibility, not unobserved intermediate numerical labels. Observed natural-pair labels remain distinct from the teacher-generated explanations that accompany them.

\end{document}